\documentclass[lp, 11pt]{elsarticle}
\makeatletter
\def\ps@pprintTitle{%
 \let\@oddhead\@empty
 \let\@evenhead\@empty
 \def\@oddfoot{}%
 \let\@evenfoot\@oddfoot}
\makeatother

\usepackage{etoolbox}
\patchcmd{\pprintMaketitle}
  {\fi\hrule}
  {\fi\ifvoid\extrainfobox\else\unvbox\extrainfobox\par\vskip10pt\fi\hrule}
  {}{}

\newsavebox\extrainfobox

\usepackage[latin1]{inputenc}
\usepackage{amsmath,bm,cite}
\usepackage{amsfonts,datetime}
\usepackage{color}
\usepackage{amssymb}
\usepackage{makeidx}
\usepackage{graphicx,subfigure}
\usepackage[left=2.5cm,right=2.5cm,top=3cm,bottom=3cm]{geometry}
\usepackage{algpseudocode, algorithm, caption}
\usepackage{setspace}
\usepackage[toc,page]{appendix}
\usepackage{lipsum}
\usepackage{array,multirow}
\usepackage{capt-of}
\usepackage{stfloats}
\usepackage{dsfont}
\usepackage{verbatim}

\newcommand{\by}{{\bf y}}
\newcommand{\bb}{{\bf b}}
\newcommand{\bv}{{\bf v}}
\newcommand{\bu}{{\bf u}}

\newcommand{\bt}{{\bf t}}
\newcommand{\bz}{{\bf z}}

\newcommand{\bx}{{\bf x}}
\newcommand{\bX}{{\bf X}}

\newcommand{\bY}{{\bf Y}}

\newcommand{\bB}{{\bf B}}

\newcommand{\bE}{{\bf E}}

\newcommand{\bid}{{\bf 1}}

\newcommand{\beps}{{\boldsymbol \epsilon}}
\newcommand{\bsig}{{\boldsymbol \Sigma}}

\newcommand{\bth}{{\boldsymbol \Theta}}

\newcommand{\bthe}{{\boldsymbol \theta}}
\newcommand{\bet}{{\boldsymbol \beta}}

\usepackage{amsthm}
\theoremstyle{definition}
\newtheorem{mydef}{Definition}
\theoremstyle{plain}
\newtheorem{mytheo}{Theorem}
\theoremstyle{plain}
\newtheorem{mylemma}{Lemma}
\theoremstyle{plain}
\newtheorem{corol}{Corollary}

  \newdateformat{mydate}{\monthname[\THEMONTH] \THEYEAR}

\DeclareMathOperator*{\argmax}{arg\,max}

\begin{document}
\begin{frontmatter}
\title{  {\LARGE Semi-parametric Order-based Generalized Multivariate Regression}}
\author{{ Milad Kharratzadeh, Mark Coates} }
\address{McGill University}
\begin{abstract}
In this paper, we  consider a generalized multivariate regression problem where the responses are monotonic functions of linear transformations of predictors. We propose a semi-parametric algorithm based on the ordering of the responses which is invariant to the functional form of the transformation function. We prove that our algorithm, which maximizes the rank correlation of responses and linear transformations of predictors, is a consistent estimator of the true coefficient matrix. We also identify the rate of convergence and show that the squared estimation error decays with a rate of $o(1/\sqrt{n})$. We then propose a greedy algorithm to maximize the highly non-smooth objective function of our model and examine its performance through extensive simulations. Finally, we compare our algorithm with traditional multivariate regression algorithms over synthetic and real data. 
\end{abstract}
\end{frontmatter}

\section{Problem Setup}
In linear multivariate regression, we have the following model:
\begin{equation}
\by_i^T = \bx_i^T \bB + \beps, \qquad i=1, \ldots, n,
\end{equation}
where $\by_i \in \mathbb{R}^{q\times 1}$ is the response vector ($q>1$), $\bx_i \in \mathbb{R}^{p\times 1}$ is the predictor vector, $\bB\in\mathbb{R}^{p\times q}$ is the coefficient matrix, and $\beps_i\in\mathbb{R}^{q\times 1}$ represents the noise with i.i.d. elements that are independent of $\bx_i$. In this paper, we consider the following extension of this problem:
\begin{equation} \label{eq:model}
\by_i^T = U_i(\bx_i^T\bB + \beps_i^T), \qquad i=1, \ldots, n, 
\end{equation}
where $U_i\!\!:\mathbb{R}\to\mathbb{R}$ is a non-degenerate monotonic function called the {\em utility} or {\em link} function. When the input of $U_i$ is a vector or a matrix, it is implied  that $U_i$ is applied separately on each individual element to give the output, which  is a vector or matrix of the same size as the input. Without loss of generality, we assume that $U_i$ is an increasing function. We propose a semi-parametric, rank-based approach to estimate $\bB$ which is invariant with respect to the functional form of  $U_i$ functions. Our approach only uses the ordering of the elements of $\by_i$, which makes it more robust to outliers and heavy-tailed noise compared to traditional regression algorithms. This also makes our approach applicable to cases where the numeric values of $\by_i$ are not available, and only their ordering is known.

We show that it is possible to consistently estimate $\bB$ solely based on the ordering of the elements of $\by_i$. Our approach to estimating $\bB$ is based on maximizing Kendall's rank correlation of $\by_i^T$ and $\bx_i^T\bB$. For notational simplicity, we assume that all the link functions are equal and denote them by $U$; however, all the results presented in this paper hold for the case where there is a separate link function, $U_i$, for each observation. Let us rewrite (\ref{eq:model}) in matrix form:
\begin{equation}\label{eq:matmodel}
\bY_{n\times q} = U(\bX_{n\times p}\bB_{p\times q}+\bE_{n\times q}), 
\end{equation}
where $p$ is the number of predictors, $q$ is the number of responses, and $n$ denotes the number of instances. $\bx_i^T$, $\by_i^T$, and $\beps_i^T$ correspond, respectively, to the $i$-th rows of $\bX$, $\bY$, and $\bE$. To find $\bB$, we propose to solve:
\begin{equation}\label{eq:opt}
\widehat{\bB}_n\!\!=\!\argmax_{\bB} \underbrace{\frac{1}{n\binom{q}{2}} \sum_{i=1}^n \sum_{j=1}^q \sum_{k=1}^q \bid(y_{ij}\!\!>\!\!y_{ik})\bid(\bx_i^T\bb_j\!\!>\!\!\bx_i^T\bb_k) }_{S_n(\bB)},
\end{equation}
where $\bb_j$ denotes the $j$-th column of $\bB$. The intuition behind this formulation is that since $U$ is increasing and the error is i.i.d. and independent of $\bx$, when we have $\bx_i^T\bb_j > \bx_i^T\bb_k$, it is more likely to have $y_{ij} > y_{ik}$ than the other way around. The term in the summation is zero for $j=k$. 
Maximizing $S_n(\bB)$ is equivalent to maximizing the average rank correlation of $\by_i^T$ and $\bx_i^T\bB$ since $2S_n(\bB)-1$ corresponds to the average over the $n$ observations of the Kendall rank correlation between $\by_i^T$ and $\bx_i^T\bB$. 

\section{Motivating Examples and Related Work}
\subsection{Learning from non-linear measurements} In many practical settings, the measurements or observations are noisy non-linear functions of a linear transformation of an underlying signal. This could be due to the uncertainties and non-linearities of the measurement device or arise from the experimental design (e.g., censoring or quantization). In the statistics and economics literature, this model is known as the single-index model and it has been studied extensively~\citep{Li89, Ichi93, Xia06, Del06, Yi15, Rad15}. The response in the single-index model is univariate and the form of the link function is sometimes assumed known unknown.

In our model, the response is a vector (which leads to a multivariate regression inference problem) and we assume that the functional form of the link function is unknown. Also, as explained in more detail below, our inference approach only uses the ordering of the elements of the response vector. Recently, it has been shown that under certain assumptions (e.g., when the predictors are drawn from a Gaussian distribution), Lasso with non-linear measurements is equivalent to one with linear measurement with an equivalent input noise proportional to the non-linearity of the link function~\citep{Thr15}. Thus, it has been suggested to use Lasso in the non-linear case as if the measurements were linear. Here, and under much more general conditions, we show that our algorithm performs better than a simple application of Lasso to the non-linear problem.

\subsection{Learning from the ordering of responses}
Our approach is particularly of interest in applications (e.g., surveys) where subjects order a set of items based on their preferences. Some examples include physicians ranking different life-sustaining therapies they are likely to remove from critically ill patients after they have already made the decision to gradually withdraw support~\citep{Chr93,Fes12}, or people ranking different types of sushi based on their preference~\citep{Kam11}. In these scenarios, the underlying model cannot be learned by traditional regression techniques, which require a numerical response. However, our algorithm is directly applicable since it only uses the ordering of the elements of the response vector.

Even in the scenarios where the actual values of responses are available (e.g., numerical ratings), it is often more sensible to focus on the ordering rather than striving to learn based on the assigned numerical values. As discussed in~\citep{Kam10}, there is often no invariant and objective mapping between true preference and observed ratings among users or survey participants, since ``each user uses his/her own mapping based on a subjective and variable criterion in his/her own mind''. Thus, the mappings might be inconsistent among different users. Moreover, the mappings might be inconsistent for a given user across different items; as noted in~\citep{Lua04}, only trained experts, e.g., wine tasters, are capable of providing consistent mappings for different items. By using the ordering in training and prediction, we minimize the effects of these inconsistencies.

\subsection{Collaborative and content-based filtering}
Our work is also related to the problem of personalized recommendation systems, but with important differences. Recommendation systems can be divided into three main categories: content--based filtering, collaborative filtering, and hybrid models; see~\citep{Lop11, Lee12, Bob13} for recent surveys. Content--based filtering employs the domain knowledge of users (e.g, demographic information and user profile data) and items (e.g., genre or actors of a movie) to predict the ratings. Collaborative filtering does not use any user or item information except a partially observed rating matrix, with rows and vectors corresponding to users and items and matrix elements corresponding to ratings. In general, the rating matrix is extremely sparse, since each user, normally, does not experience and rate all items. Hybrid systems combine collaborative and content--based filtering, e.g., by making separate predictions with each filtering approach and averaging the results.

If the regression--based framework described in this paper were used in a recommendation system, it would predict each user's ordering of a set of items based on a set of features for that user. These features could include demographic information, user profile data, or ratings of a fixed set of items. Contrary to content--based filtering, our approach does not need domain-specific knowledge about the features of items (e.g. relevant features are different for books and movies). This is potentially useful in applications where the items to be ranked are diverse in nature - for example, products on the online Amazon store. Also, as opposed to collaborative filtering, we can incorporate user profile data and provide predictions for new users even if they have provided no prior ratings (i.e., providing predictions only based on features such as demographic data). In Section~\ref{sec:real}, we provide a preference prediction task where neither collaborative nor content--based filtering is applicable since we want to make predictions for new users without using domain-specific knowledge about the items. Thus, our algorithm is different from the problems of collaborative and content-based filtering. It is important to stress that we introduce and study a general semi-parametric multivariate regression method which can be used in recommendation systems, but this is just one of multiple potential applications. 

\subsection{Maximum rank correlation estimation} 
In~\citep{Han87}, Han considered a problem similar to~(\ref{eq:model}) but with an important difference. His formulation, called Maximum Rank Correlation (MRC) estimation, was stated for the multiple regression setting, where $y_i$ is real-valued rather than a $q$-vector, and his goal was to maximize the rank correlation across instances. Therefore, the goal in MRC estimation is to capture the ordering of $y_i, i=1,\ldots,n$ (across instances), whereas in this paper,  our goal is to capture the ordering of $y_{ij}, j=1,\ldots,q$ for a fixed $i$ (for a specific instance, across responses). 
This difference significantly changes the scope of the problem and its theoretical properties. Considering the ordering across responses enables us to model  applications where an instance's ordering (or rating) of a set of items depends exclusively on its predictors.  Also, as we see in the next section, the identifiability and consistency conditions for problem~(\ref{eq:model}) differ significantly from those of the multiple regression problem. 

There are extensions of the MRC approach (e.g., \citep{Cav98, Abr03}), but they all are in the multiple regression domain and only differ in how they define the objective function to solve the same problem. Our work differs from them for the same reasons mentioned above. Unlike exploded and ordered logit models~\citep{All94,Ala13}, our approach is semi-parametric and invariant to the functional form of $U$. Our work is also markedly different from the `learning to rank' problem in the information retrieval literature, where the goal is to find relevant documents for a given query~\citep{Cao07,Liu11}.


\section{Summary of Results}
In Section~\ref{sec:theory}, we outline the identifiability conditions on the true model. We show that the true matrix, $\bB^*$, is identifiable up to a scaling and addition of a vector to all its columns. Moreover, we show that the maximizer of the optimization problem in~(\ref{eq:opt}) is a consistent estimator of the true matrix. In Section~\ref{sec:rate}, we study the convergence rate and show that our estimator's squared error decays faster than $o(1/\sqrt{n})$ if the Hessian of $E[S_n(\bB)]$ is negative definite at $\bB^*$. In Section~\ref{sec:opt}, we provide a greedy algorithm to estimate the maximizer of~(\ref{eq:opt}) in polynomial time and show how our algorithm can be extended to provide sparsity among the elements of $\bB$. Finally, in Sections~\ref{sec:sim1} and~\ref{sec:real}, we provide extensive experimental results using both synthetic and real datasets and show that our algorithm is successful in estimating the underlying true models and provides better prediction results compared to other applicable algorithms.

\section{Strong Consistency}\label{sec:theory}

In this section, we show that the solution of~(\ref{eq:opt}) is strongly consistent under certain conditions. 
$S_n(\bB)$ is invariant to the multiplication of all elements of $\bB$ by a positive constant; i.e., for $c>0$, $S_n(\bB) = S_n(c\bB)$. The objective function also does not change if the same vector is added to all of the columns of $\bB$; i.e., for any $\bet \in \mathbb{R}^{p\times1}$, $S_n(\bB) = S_n(\bB + \bet\mathds{1}_{1\times q})$, where $\mathds{1}$ is a vector of all ones. These invariances are expected, since to have a semi-parametric estimate, we target maximizing the rank correlation and ranks are not affected when all the elements are  multiplied by a positive constant ($c$), or are increased/decreased by the same amount ($\bx^T\bet$). In other words, since our  estimation is semi-parametric in $U$ (and thus, must be invariant to strictly monotonic transformations of observations),  $\bB$ and $c\bB + \bet\mathds{1}_{1\times q}$ are equivalent.  So, we assume that $\|\bB\|_F = 1$ (normalization), and that the last column of $\bB$ is all zeros (subtracting the last column from all columns). We perform the optimization in~(\ref{eq:opt}) over the set $$\mathcal{B} \triangleq \{\bB_{p \times q}: \|\bB\|_F = 1 \text{ and } \bB_{i,q} = 0 \text{ for } i=1,2,\ldots, p\},$$ where $\|\cdot\|_F$ denotes the Frobenius norm. 

 Let us denote the true coefficient matrix by $\bB^*$ and without loss of generality assume $\bB^*\in\mathcal{B}$; otherwise, we can find $c>0$ and $\bet\in\mathbb{R}^{p\times 1}$ such that $c\bB^* + \bet\mathds{1}_{1\times q} \in\mathcal{B}$ which gives the same value for objective function as $\bB^*$. Also, to have a non-degenerate problem, we assume that $\bB^*$ does not have rank 1, since in that case there exist two vectors $\bu\in \mathbb{R}^{p\times 1},\bv \in \mathbb{R}^{q\times 1}$ such that $\bB^* = \bu\bv^T$, and $\by^T = U(\bx^T\bu\bv^T+\beps)$. Therefore, the ordering of the elements of $\by$ will be either the same as the ordering of the elements of $\bv$ (if $\bx^T\bu>0$) or the reverse of it (if $\bx^T\bu<0$) with perturbations due to the noise. So, different observed orderings of the elements of $\by$ are caused merely by noise which is not of interest.
 Finally, we assume that no two columns of $\bB^*$ are equal, because in that case the expected values of the corresponding elements of $\by$ will be the same.
 
 Given the model in~(\ref{eq:matmodel}), to prove strong consistency, we need the following three conditions:
\begin{itemize}
\item [{\bf (C1)}] $U$ is a non-degenerate increasing function and changes value at least at one non-zero point (i.e., $U$ is not a step function changing value only at 0).
\item [{\bf (C2)}] The elements of $\bE$ are i.i.d. random variables. 
\item [{\bf (C3)}] The rows of $\bX$ are i.i.d. random vectors of size $p$, independent of the elements of $\bE$, and have a distribution function $F_X$ such that:
\begin{itemize}
\item[{\bf (C3.1)}] the support of $F_X$ is not contained in any proper linear subspace of $\mathbb{R}^p$, and
\item[\bf{(C3.2)}] for all $j\in\{1,2,\ldots,q\}$ the conditional distribution of $x_j$ given the other components has everywhere positive Lebesgue density. 
\end{itemize}
\item[{\bf (C4)}] $\bB^*$ is an interior point of $\mathcal{B}$. Moreover, $\bB^*$ has a rank higher than one and no two columns of it are equal.
\end{itemize}

As we show later, the second part of condition {\bf (C1)} is needed in the proof of identifiability. However, for all practical purposes, the step function at 0 can be replaced by an approximate function changing value over $[-\epsilon, \epsilon]$ for some $\epsilon>0$, for which the theoretical results we present in the following hold. Conditions {\bf (C3.1)} and {\bf (C3.2)} are also required for identifiability, and hold in many settings; e.g., when the rows of $\bX$ have a multivariate Gaussian distribution. 
In {\bf (C4)}, $\bB^* \in \mathcal{B}$ implies that its last column is all zeros, and because no two columns are equal, we can conclude that every column except the last one has at least one non-zero element. For some known constant $\eta>0$ which is less than all the absolute values of these non-zero elements, define 
\begin{align}\mathcal{B_{\eta}} \triangleq \{&\bB:  \bB\in\mathcal{B} ; \text{ and }   \forall j\in\{1,\ldots,p\} \quad  \exists  i \in\{1,\ldots,p\} \text{ s.t. } |\bB_{i,j}|\geq \eta\}.\nonumber
\end{align}
 Denoting the solution of~(\ref{eq:opt}) over the set $\mathcal{B_{\eta}}$ by $\widehat{\bB}_n$, we prove:
\begin{equation}
\lim_{n\to \infty} \widehat{\bB}_n \to \bB^*.
\end{equation}  

We conduct the proof in three steps. In Lemma 1, we prove the identifiability; in Lemma 2, we prove the convergence of $S_n(\bB)$ to the expected value of the rank correlation; and finally, we prove the consistency in Theorem~1.
\begin{mylemma}[Identifiability]
Given {\bf(C1)}---{\bf (C4)}, $\bB^*$ attains the unique maximum of $E\left[S_n(\bB))\right]$ over the set $\mathcal{B}$.
\end{mylemma}

\begin{proof}
For a given $\bx$ such that $\bx^T\bb_j^* >\bx^T\bb_k^*$, we have:
\begin{equation}
P_{\beps|\bx}(y_j>y_k)\geq P_{\beps|\bx}(y_k>y_j), 
\end{equation}  
where $\by^T = U(\bx^T\bB^* + \beps^T)$, and $\bb_j^*$ is the $j$-th column of $\bB^*$.  For any matrix $\bB\in\mathcal{B}$, we have:
\vspace*{-0.05 in}
\begin{align}
&E_{\bx,\beps}\Big[\bid(y_{j}>y_{k})\bid(\bx^T\bb_j > \bx^T\bb_k)  +\bid(y_{j}<y_{k})\bid(\bx^T\bb_j < \bx^T\bb_k)\Big]  \label{eq:maxim} \\ = &\quad E_{\bx}\Big[P_{\beps|\bx}(y_{j}>y_{k})\bid(\bx^T\bb_j > \bx^T\bb_k) + P_{\beps|\bx}(y_{j}<y_{k})\bid(\bx^T\bb_j < \bx^T\bb_k)\Big]. \nonumber
\end{align}
For any $j$ and $k$, $\bb^*_j$ and $\bb^*_k$ maximize this expected value, because for any given $\bx$, the larger term between $P_{\beps|\bx}(y_{j}>y_{k})$ and $P_{\beps|\bx}(y_{j}<y_{k})$ is chosen in the expected value. Therefore, $\bB^*$ maximizes the expected value of each of the terms in~(\ref{eq:opt}) and consequently, maximizes $E_{\bx,\beps}\left[S_n(\bB)\right]$. Next, we show that $\bB^*$ is the unique maximizer of $E_{\bx,\beps}\left[S_n(\bB)\right]$ over the set $\mathcal{B}$. 

For any $\widetilde{\bB}\in\mathcal{B}, \widetilde{\bB} \neq \bB^*$, we show that
\vspace*{-0.05 in}
 \begin{equation}
E_{\bx,\beps}\left[S_n(\bB^*)\right] - E_{\bx,\beps}\left[S_n(\widetilde{\bB})\right]>0.
\end{equation}
For any pair of $j,k\in\{1,2,\ldots,q\}, k\neq j$, we can define the following two sets:
\vspace*{-0.05 in}
\begin{align*}
\mathcal{D}_1 & \triangleq \{\bx\in\mathcal{R}^p: \bx^T(\bb^*_j - \bb^*_k)>0 \text{ and } \bx^T(\widetilde{\bb}_j-\widetilde{\bb}_k)<0\}, \\
\mathcal{D}_2 & \triangleq \{\bx\in\mathcal{R}^p: \bx^T(\bb^*_j - \bb^*_k)<0 \text{ and } \bx^T(\widetilde{\bb}_j-\widetilde{\bb}_k)>0\}.
\end{align*}
If $(\bb^*_j - \bb^*_k) \neq c (\widetilde{\bb}_j-\widetilde{\bb}_k)$ for some $c>0$, then both $\mathcal{D}_1$ and $\mathcal{D}_2$ have Lebesgue measure greater than 0. We claim that given $\widetilde{\bB}\in\mathcal{B}, \widetilde{\bB} \neq \bB^*$ and {\bf (C4)}, a pair $j$ and $k$ such that $(\bb^*_j - \bb^*_k) \neq c (\widetilde{\bb}_j-\widetilde{\bb}_k)$ exists. 
We prove this by contradiction. Assume that such a pair does not exist. Then, since the last columns of both $\widetilde{\bB}$ and $\bB^*$ are zero, setting $k=q$ implies that there exists $c_i>0$ such that $\bb^*_i = c_i \widetilde{\bb}_i$ for $i=1,\ldots,q-1$. Also, for any $j, k \in\{1,\ldots,q-1\}$, there must exist a $c_{j,k}>0$ such that $(\bb^*_j - \bb^*_k) = c_{j,k} (\widetilde{\bb}_j-\widetilde{\bb}_k)$. Combining these two, we get $(1-c_{j,k}/c_j)\bb^*_j = (1-c_{j,k}/c_k)\bb^*_k$. Since this holds for all $j$ and $k$, we can conclude that all columns of $\bB^*$ are multiples of each other and consequently, $\bB^*$ has rank 1, which is a contradiction. Therefore, there exists a pair $j$ and $k$ for which both $\mathcal{D}_1$ and $\mathcal{D}_2$ have Lebesgue measure greater than 0.

Now, define the following two sets:
\begin{align}
\mathcal{G}_1 & \triangleq \Big\{\bx\in\mathcal{R}^p: \bx^T(\bb^*_j - \bb^*_k)>0 \text{ and }  E[U(\bx^T\bb^*_j+\beps_j)] > E[U(\bx^T\bb^*_k+\beps_k)]\Big\}, \label{eq:G1}\\
\mathcal{G}_2 & \triangleq \Big\{\bx\in\mathcal{R}^p: \bx^T(\bb^*_j - \bb^*_k)<0 \text{ and } \nonumber E[U(\bx^T\bb^*_j+\beps_j)] < E[U(\bx^T\bb^*_k+\beps_k)]\Big\}.
\end{align}
Next, we show that $\mathcal{H}_1 = \mathcal{D}_1 \cap \mathcal{G}_1$ and/or $\mathcal{H}_2 = \mathcal{D}_2 \cap \mathcal{G}_2$ have Lebesgue measure greater than 0. This is trivial if $U$ is strictly increasing, since $\mathcal{H}_1 = \mathcal{D}_1$ and $\mathcal{H}_2 = \mathcal{D}_2$.  We show that $\mathcal{H}_1$ and/or $\mathcal{H}_2$ have positive measure in general. Take an $\bx\in\mathcal{D}_1$; it is clear that for any $\alpha>0$, $\alpha\bx$ is also in $\mathcal{D}_1$, and $-\alpha\bx$ is in $\mathcal{D}_2$. If we change $\alpha$ from $0^+$ to $+\infty$, then $\alpha\bx^T(\bb^*_j - \bb^*_k)$ changes from $0^+$ to $+\infty$ and $-\alpha\bx^T(\bb^*_j - \bb^*_k)$ changes from $0^-$ to $-\infty$. 
Since $U$ is non-degenerate and changes value at a non-zero point, there exists a neighborhood $\mathcal{A}$ such that for $\alpha\in\mathcal{A}$, we have $E\left[U(\alpha\bx^T\bb^*_j+\beps_j)\right] > E\left[U(\alpha\bx^T\bb^*_k+\beps_k)\right]$ and/or  $E\left[U(-\alpha\bx^T\bb^*_j+\beps_j)\right] < E\left[U(-\alpha\bx^T\bb^*_k+\beps_k)\right]$. Thus $\mathcal{H}_1$ and/or  $\mathcal{H}_2$ have Lebesgue measure greater than 0.

Without loss of generality, assume that $\mathcal{H}_1$ defined for $j=j'$ and $k=k'$ has Lebesgue measure greater than 0. In the following we show that $E[S_n(\bB^*)] >E[S_n(\widetilde{\bB})]$, which proves the lemma. We have:
\begin{align}
n\binom{q}{2} \cdot & E_{\bx,\beps}\big[S_n(\bB^*)-S_n(\widetilde{\bB})\big] \nonumber \\ 
 = & \sum_{i,j,k} E_{\bx,\beps}\Big[\bid(y_{ij}\!\!>\!\!y_{ik})\Big(\bid(\bx_i^T\bb^*_j\!\!>\!\!\bx_i^T\bb^*_k)-\bid(\bx_i^T\widetilde{\bb}_j\!\!>\!\!\bx_i^T\widetilde{\bb}_k)\Big) \nonumber \\
 & \qquad + \bid(y_{ij}\!\!<\!\!y_{ik})\Big(\bid(\bx_i^T\bb^*_j\!\!<\!\!\bx_i^T\bb^*_k)-\bid(\bx_i^T\widetilde{\bb}_j\!\!<\!\!\bx_i^T\widetilde{\bb}_k)\Big)\Big] \label{eq:pairs}  \\ 
 \geq & \sum_i P(\bx_i\!\in\!\mathcal{H}_1) E_{\bx,\beps}\Big[\bid(y_{ij'}>y_{ik'})  - \bid(y_{ij'}<y_{ik'})\Big|\bx_i\in\mathcal{H}_1 \Big]  \label{eq:cond}\\
= &  \sum_i P(\bx_i\!\in\!\mathcal{H}_1) E_{\bx}\Big[P_{\beps|\bx_i}(y_{ij'}>y_{ik'})\!-\!  P_{\beps|\bx_i}(y_{ij'}<y_{ik'})\Big|\bx_i\in\mathcal{H}_1 \Big]\!\!>\!\!0\label{eq:cond2}
\end{align}
The summation in~(\ref{eq:pairs}) is over all distinct pairs, $j$ and $k$. The inequality in~(\ref{eq:cond}) is based on the argument made for~(\ref{eq:maxim}); for any given $i$, $j$, and $k$, the expectation in~(\ref{eq:pairs}) is greater than or equal to zero (because $\bB^*$ maximizes~(\ref{eq:maxim}) for every $\bx$ and any pair $j$ and $k$). Therefore, by removing all pairs of $j$ and $k$ such that $j\neq j'$ and $k \neq k'$, and conditioning only on $\bx_i\in\mathcal{H}_1$, we get a smaller term. The first term in~(\ref{eq:cond2}) is strictly greater than zero because $\mathcal{H}_1$ has positive measure, and hence $P(\bx_i\!\in\!\mathcal{H}_1)>0$. The second term  is positive, because for any $\bx_i\in\mathcal{G}_1$, the term inside $E_{\bx}$ is strictly greater than zero.
\end{proof}

\begin{mylemma}[Convergence]
Given {\bf(C1)}---{\bf (C4)}, and denoting $$h_i(\bB) = \frac{1}{\binom{q}{2}} \sum_{j=1}^q \sum_{k=1}^q \bid(y_{ij}>y_{ik})\bid(\bx_i^T\bb_j > \bx_i^T\bb_k),$$
we have: $$ S_n(\bB) \stackrel{a.s.}{\xrightarrow{\hspace*{0.5cm}}} E\left[h_i(\bB)\right]$$
\end{mylemma}

\begin{proof}
Define 
\begin{align*}
 & h(\bB) \triangleq  E\left[h_i(\bB)\right] \quad ; \quad D_{\delta}(\bB) \triangleq  \{\bet\in\mathbb{R}^{p\times q}: \bet \in \mathcal{B}, \|\bet-\bB\|_F< \delta\} \\ &
 \overline{g_i}(\bB,\delta) \triangleq  \sup_{\bet \in D_{\delta}(\bB)}  \{h_i(\bet) - h(\bet)\} \quad ; \quad \underline{g_i}(\bB,\delta) \triangleq \inf_{\bet \in D_{\delta}(\bB)}  \{h_i(\bet) - h(\bet)\} \\
&   \overline{g}(\bB,\delta) \triangleq  E\left[\overline{g_i}(\bB,\delta)\right] \quad ; \quad \underline{g}(\bB,\delta) \triangleq  E\left[\underline{g_i}(\bB,\delta)\right] 
\end{align*}
It is clear that $h_i(\bB)$ is bounded and moreover, given {\bf (C3)}, $h_i(\bB)$, and consequently $h(\bB)$, are continuous in $\bB\in \mathcal{B_{\eta}}$ almost surely. Therefore, since $\mathcal{B_{\eta}}$ is compact and separable, for any $\bB\in\mathcal{B_{\eta}}$, there exists a sequence $\{\bB_t\}$ in a countable dense subset of $\mathcal{B_{\eta}}$ such that:
\begin{equation}
\lim_{t\to\infty} h_i(\bB_t) = h_i(\bB) \quad ; \quad  \lim_{t\to\infty} h(\bB_t) = h(\bB). \nonumber
\end{equation}
Thus, almost surely:
\begin{equation}
\lim_{\delta\to 0} \overline{g_i}(\bB,\delta) = h_i(\bB) - h(\bB) \quad ; \quad   \lim_{\delta\to 0} \underline{g_i}(\bB,\delta) = h_i(\bB) - h(\bB). \nonumber
\end{equation}
Taking the expected value of both sides of these limits, we get for all $\bB\in\mathcal{B}_{\eta}$:
\begin{equation} \label{eq:lim0}
\lim_{\delta\to 0} \overline{g}(\bB,\delta) = 0; \qquad \qquad \lim_{\delta\to 0} \underline{g}(\bB,\delta) = 0. 
\end{equation}
We use the Borel--Cantelli lemma~\citep{Kle13} to prove the lemma.
\begin{align}
& \sum_{n=1}^{\infty} P\left(\max_{\bB\in\mathcal{B}_{\eta}} \left|S_n(\bB) - h(\bB)\right|>\epsilon\right) \nonumber \\
& = \sum_{n=1}^{\infty} P\left(\max_{\bB\in\mathcal{B}_{\eta}} \left|\frac{1}{n} \sum_{i=1}^n \left(h_i(\bB) - h(\bB)\right)\right|>\epsilon\right) \label{eq:conv1}\\
& \leq \sum_{n=1}^{\infty} P\left(\left|\frac{1}{n}\sum_{i=1}^n \sup_{\bB\in\mathcal{B}_{\eta}} \left(h_i(\bB) - h(\bB)\right)\right| > \epsilon \right) \label{eq:conv3} + \sum_{n=1}^{\infty} P\left(\left|\frac{1}{n}\sum_{i=1}^n \inf_{\bB\in\mathcal{B}_{\eta}} \left(h_i(\bB) - h(\bB)\right)\right| > \epsilon \right).  
\end{align}
$\mathcal{B}_{\eta}$ is a compact set and also $\mathcal{B}_{\eta} \subset \mathop{\bigcup}_{\bet\in\mathcal{B}_{\eta}} D_{\delta}(\bet)$ for any small $\delta$. However, for a compact set, every open cover has a finite subcover. Thus, for a finite $L$, there exist $\bet_1, \bet_2, \ldots, \bet_L$, where $\bet_l\in\mathcal{B}_{\eta}$, and:
\begin{equation}
\mathcal{B}_{\eta} \subset \mathop{\bigcup}_{l=1}^L D_{\delta_l}(\bet_l)
\end{equation}
Also, assume that for $l=1, \ldots, L$, we have $\left|\overline{g}(\bet_l,\delta_l)\right| < \epsilon/2$ and $\left|\underline{g}(\bet_l,\delta_l)\right| < \epsilon/2$. This is possible to achieve because of (\ref{eq:lim0}). Now, from (\ref{eq:conv1})--(\ref{eq:conv3}), we have:
\begin{align}
 & \sum_{n=1}^{\infty} P\left(\max_{\bB\in\mathcal{B}_{\eta}} \left|S_n(\bB) - h(\bB)\right|>\epsilon\right) \\&  \leq  \sum_{n=1}^{\infty}\sum_{l=1}^L P\left(\left|\frac{1}{n}\sum_{i=1}^n \overline{g_i}(\bet_l,\delta_l) \right| > \epsilon \right) + \sum_{n=1}^{\infty}\sum_{l=1}^L P\left(\left|\frac{1}{n}\sum_{i=1}^n \underline{g_i}(\bet_l,\delta_l) \right| > \epsilon \right) \nonumber \\
& \leq  \sum_{n=1}^{\infty}\sum_{l=1}^L P\left(\left|\frac{1}{n}\sum_{i=1}^n \overline{g_i}(\bet_l,\delta_l)-\overline{g}(\bet_l,\delta_l) \right| > \epsilon/2 \right) \label{eq:BC2} \\ 
& \quad \quad + \sum_{n=1}^{\infty}\sum_{l=1}^L P\left(\left|\frac{1}{n}\sum_{i=1}^n \underline{g_i}(\bet_l,\delta_l)-\underline{g}(\bet_l,\delta_l) \right| > \epsilon/2 \right) \nonumber.
\end{align}
For each $l$, we can invoke the Small Law of Large Numbers for {\it U}-statistics~\citep[Chapter 5]{Ser09} to get:
\begin{align}
\frac{1}{n}\sum_{i=1}^n \overline{g_i}(\bet_l,\delta_l)\stackrel{a.s.}{\xrightarrow{\hspace*{0.5cm}}}\overline{g}(\bet_l,\delta_l)  \quad ; \quad 
\frac{1}{n}\sum_{i=1}^n \underline{g_i}(\bet_l,\delta_l)\stackrel{a.s.}{\xrightarrow{\hspace*{0.5cm}}}\underline{g}(\bet_l,\delta_l)\label{eq:asconverg}
\end{align}
In the following, we use the Borel-Cantelli lemma~\citep{Kle13} twice to get the desired result.  Reordering  the summands,  we can re-write the first term in~(\ref{eq:BC2}) as follows: $$\sum_{l=1}^L \underbrace{\sum_{n=1}^{\infty} P\left(\left|\frac{1}{n}\sum_{i=1}^n \overline{g_i}(\bet_l,\delta_l)-\overline{g}(\bet_l,\delta_l) \right| > \epsilon/2 \right)}_{(\star)}.$$ According to the Borel-Cantelli lemma, the almost sure convergence in~(\ref{eq:asconverg}) implies that each of terms marked by $(\star)$ is finite, and thus their sum is finite (since $L$ is finite). The same argument can be made for the second term in (\ref{eq:BC2}). Therefore, $\sum_{n=1}^{\infty} P\left(\max_{\bB\in\mathcal{B}_{\eta}} \left|S_n(\bB) - h(\bB)\right|>\epsilon\right)$ is less than something finite, and hence, is finite. Applying the Borel-Cantelli result proves the lemma.
\end{proof}

\begin{mytheo}
Given {\bf(C1)}---{\bf (C4)}, the solution of~(\ref{eq:opt}) over the set $\mathcal{B_{\eta}}$, $\widehat{\bB}_n$, is strongly consistent; i.e.,  $$ \widehat{\bB}_n \to \bB^* \quad \text{almost surely}.$$
\end{mytheo}
\begin{proof}
Given the results of Lemmas 1 and 2, we are now ready to prove the consistency of the solution of~(\ref{eq:opt}) over $\mathcal{B}_{\eta}$. We do so by showing that any set, $\mathcal{B}_0\subset\mathbb{R}^{p\times q}$, that contains $\bB^*$, also contains $\widehat{\bB}_n$ as $n\to\infty$. 

Define $\mathcal{B}_1 \triangleq \mathcal{B}_{\eta} - (\mathcal{B}_0 \cap \mathcal{B}_{\eta})$. $\mathcal{B}_1$ is compact and there exists $\zeta = h(\bB^*) - \max_{\bB \in \mathcal{B}_1} h(\bB)$ which is always greater than 0 (because $\bB^*$ attains the unique maximum and $\bB^*\not\in \mathcal{B}_1$ ). From the result of Lemma~2, we know that for any $\zeta$, there is  an $N$, such that for $n>N$, $|S_n(\bB) - h(\bB)|<\zeta/2$ for all $\bB\in\mathcal{B}_{\eta}$ with probability 1. This implies that $\widehat{\bB}_n$ cannot be in $\mathcal{B}_1$; because otherwise, we get $h(\bB^*) - S_n(\widehat{\bB}_n) > \zeta/2$ which is in contradiction with almost sure convergence. Thus, $\widehat{\bB}_n \in \mathcal{B}_0$ with probability 1, and since this is true for any $\mathcal{B}_0$, we have $ \widehat{\bB}_n \to \bB^*$ almost surely.
\end{proof}

\section{Rate of Convergence}\label{sec:rate}
For ease of notation, let $\bthe \in \mathbb{R}^{p(q-1)}$ be the vectorization of the matrix $\bB\in\mathcal{B}$, except the last column which is assumed to be all zero. Thus,
\begin{equation}
\bthe \triangleq (B_{1,1},B_{2,1},\ldots, B_{p,1}, \ldots, B_{1,q-1}, \ldots, B_{p,q-1}).
\end{equation}
For $\bB\in\mathcal{B}$, the corresponding $\bthe$ is in $\bth$, the set of $d$-dimensional vectors with norm 1. So, we can denote $\bB$ and its columns as functions of $\bthe$, and write:
\begin{align}
h(\bz,\bthe) = & \sum_{j=1}^q \sum_{k=1}^q \bid(y_{j}\!\!>\!\!y_{k})\bid(\bx^T\bb_j(\bthe)\!\!>\!\!\bx^T\bb_k(\bthe)),
\end{align}
where $\bz = (\by, \bx)\in \mathbb{R}^{p+q}$ is the joint vector of predictors and responses for an instance, and $\bb_j$ denotes the $j$'th column of $\bB$.
Let $\widehat{\bthe}_n$ correspond to $\widehat{\bB}_n$ and $\bthe_0$ correspond to $\bB^*$.
Based on {\bf (C4)} we know that $\bthe_0$ is an interior point of $\bth$.
In the previous section, we showed that $\widehat{\bthe}_n  \stackrel{a.s.}{\xrightarrow{\hspace*{0.5cm}}} \bthe_0$. 
In this section, we study the rate of convergence and show that $\|\widehat{\bthe}_n-\bthe_0\|^2 \leq o_p(1/\sqrt{n})$, where $\| \cdot \|$ is the Euclidean norm.

The following Lemma plays a critical role in establishing the rate of convergence. Its proof, using results from~\citep{ Pakes, Nolan, Sherman}, is included in the Appendix. 
\begin{mylemma}\label{lem:euc}
For $\bthe$ in an $o_p(1)$ neighborhood of $\bthe_0$, and $S(\bthe) \triangleq E_{\bz}\left[h(\bz,\bthe)\right]$, we have:
\begin{equation}
S_n(\bthe) = S(\bthe) + S_n(\bthe_0) - S(\bthe_0) + o_p(1/\sqrt{n}).
\end{equation} 
\end{mylemma}
\begin{proof}
See the Appendix.
\end{proof}

For the next Theorem we require that there exists an $o_p(1)$ neighborhood $\mathcal{A}$ of $\bthe_0$ and a constant $\kappa>0$ for which $S(\bthe)-S(\bthe_0) \leq -\kappa \|\bthe-\bthe_0\|^2$ for all $\bthe\in \mathcal{A}$. Assume $\nabla$ and $\nabla_2$ denote the first and second order derivatives with respect to $\bthe$. The existence of $\mathcal{A}$ is guaranteed since $\bthe_0$ is an interior point of $\bth$.  Also, $\kappa>0$ exists if $\nabla_2 S(\bthe_0)$ is negative definite, because $\nabla S(\bthe_0)={\bf 0}$ (since $S$ is maximized at $\bthe_0$) and in $o_p(1)$ neighborhood of $\bthe_0$, Taylor expansion gives us:
\begin{equation}
S(\bthe)-S(\bthe_0) = \frac{1}{2}(\bthe-\bthe_0)^T \nabla_2S(\bthe_0)(\bthe-\bthe_0) + o_p(\|\bthe-\bthe_0\|^2).
\end{equation}

Since $\nabla_2 S(\bthe_0)$ is negative definite, there exists a positive $\alpha$ such that $\frac{1}{2}(\bthe-\bthe_0)^T \nabla_2S(\bthe_0)(\bthe-\bthe_0) \leq -\alpha \|\bthe-\bthe_0\|^2 $. In an $o_p(1)$ neighborhood of $\bthe_0$, setting $\kappa = (1-\epsilon)\alpha$ for any small positive $\epsilon$ gives $S(\bthe)-S(\bthe_0) \leq -\alpha \|\bthe-\bthe_0\|^2 + o_p(\|\bthe-\bthe_0\|^2) \leq -\kappa \|\bthe-\bthe_0\|^2$. In sum, a sufficient condition for the existence of $\kappa>0$ is the following:
\begin{itemize}
\item[{\bf (C5)}] The matrix $\nabla_2 S(\bthe_0)$ is negative definite.
\end{itemize}

\begin{mytheo}
Assume that there exists an $o_p(1)$ neighborhood $\mathcal{A}$ of $\bthe_0$ and a constant $\kappa>0$ for which $S(\bthe)-S(\bthe_0) \leq -\kappa \|\bthe-\bthe_0\|^2$ for all $\bthe\in \mathcal{A}$. Then, the squared estimation error decays with a rate faster than  $1/\sqrt{n}$: $\|\widehat{\bthe}_n-\bthe_0\|^2 \leq o_p(1/\sqrt{n})$.
\end{mytheo}
\begin{proof}
By definition of $\widehat{\bthe}_n$ we have:
\begin{equation}
0 \leq S_n(\widehat{\bthe}_n) - S_n(\bthe_0).
\end{equation}
Rewriting this inequality using the result of Lemma~\ref{lem:euc} gives:
\begin{align}\label{eq:op}
0 & \leq S(\widehat{\bthe}_n) - S(\bthe_0) + o_p(1/\sqrt{n}) \leq -\kappa \|\widehat{\bthe}_n-\bthe_0\|^2 + o_p(1/\sqrt{n}).
\end{align}
which gives us $\|\widehat{\bthe}_n-\bthe_0\|^2 \leq o_p(1/\sqrt{n})$.
\end{proof}

\begin{corol} Given {\bf(C1)}---{\bf (C5)},  we have $\|\widehat{\bthe}_n-\bthe_0\|^2 \leq o_p(1/\sqrt{n})$.
\end{corol}

\section{Optimization Algorithm}\label{sec:opt}
In the previous section, we showed that solving~(\ref{eq:opt}) provides a consistent estimate of $\bB^*$. However, the objective function is very non-smooth (the sum of many step functions) and finding $\widehat{\bB}_n$ can be challenging. In this section, we propose a fast, greedy algorithm to solve~(\ref{eq:opt}). 

First, consider the following maximization problem:
\begin{equation}\label{eq:opt_simp}
\widehat{x} = \argmax_{x \in \mathbb{R}} \qquad \sum_{t=1}^T \bid(u_t + v_tx>0),
\end{equation}
where $\{u_t\}_{t=1}^T$ and $\{v_t\}_{t=1}^T$ are given real numbers. This objective function is a piece-wise constant function changing values at $\{-u_t/v_t\}_{t=1}^T$ (one step function changes value at each of these points). Therefore, this function takes $O(T)$ different values. Computing each of these values requires $O(T)$ operations, and thus, finding $\widehat{x}$ by computing all the possible function values requires $O(T^2)$ operations. 

We propose an $O(T\log{T})$ algorithm to find $\widehat{x}$. The algorithm works as follows. First, we sort the sequence $\{-u_t/v_t\}_{t=1}^T$, in $O(T\log{T})$. 
Then, we start $x$ from a value less than the first sorted point, i.e., $\min\{-u_t/v_t; t=1,\ldots,T\}$, and at each step, move forward to the next smallest point.
 At each step, we cumulatively add or subtract 1 depending on the sign of $v_t$ (i.e., depending on whether one of the step functions went from 0 to 1 or vice-versa) and keep track of the largest cumulative sum seen so far, and the value of $x$ for which that maximum happened.
After going through all $T$ points, the largest observed cumulative sum  is equal to the maximum value of the objective function, and the corresponding value of $x$ is  $\widehat{x}$. Since the objective function is piece-wise constant, its maximum is attained over an interval; we set $\widehat{x}$ to the center of this interval. Therefore, we can solve~(\ref{eq:opt}) in $O(T\log{T}) + O(T)$ which is equivalent to $O(T\log{T})$. The details of the algorithm are summarized below under Algorithm~1. We use this algorithm repeatedly to solve~(\ref{eq:opt}). 

\begin{algorithm}[!h]
\caption*{ \bf Algorithm 1 --- for solving~(\ref{eq:opt_simp}) }
\begin{algorithmic}
\State {\bf Inputs}: $u_t$, $v_t$, $t=1,2,\ldots,T$ ;  {\bf Output}: $\widehat{x}$
\State \texttt{Compute:} $r_t = -u_t/v_t, t=1,2,\ldots,T$
\State \texttt{Sort} $r_t: r_{(1)} \leq r_{(2)} \leq \ldots \leq r_{(T)}$ 
\State \Comment{\texttt{Notation:} $v_{(i)}$: corresponding $v$ of $r_{(i)}$}
\For{\texttt{$i=1,2,\ldots,T+1$}}
    \If{$i=1$}
    \State $\widehat{x} = r_{(1)} - 1$
    \State $ m \gets \sum_{t=1}^T \bid(u_t + v_t\widehat{x}>0)$
    \State $ s \gets \sum_{t=1}^T \bid(u_t + v_t\widehat{x}>0)$
    \ElsIf{$i=T+1$}
    \State $s\gets s + \text{sign}(v_{(i)})$ 
    \If{$s>m$}
    \State $\widehat{x} = r_{(T)} + 1$
    \State $m\gets s$
    \EndIf
    \Else
        \State $s\gets s + \text{sign}(v_{(i)})$
    \If{$s>m$}
    \State $\widehat{x} = (r_{(i-1)} + r_{(i)}) /2$
    \State $m\gets s$
    \EndIf
    \EndIf
\EndFor
\end{algorithmic}
\end{algorithm}

We use an alternating maximization scheme to solve~(\ref{eq:opt}). We go through the $p\times q$ elements of $\bB$ one-by-one and update them to maximize $S_n(\bB)$ while the other elements are kept fixed. We show that each of these optimization problems are of the form~(\ref{eq:opt_simp}) and can be solved easily. Assume that we want to maximize $S_n(\bB)$ while all elements except $B_{rs}$ are fixed. Remember that:
\begin{equation*}\label{eq:opt_mod}
S_n(\bB) = \frac{1}{n\binom{q}{2}} \sum_{i=1}^n \sum_{j=1}^q \sum_{k=1}^q \bid(y_{ij}\!\!>\!\!y_{ik})\bid(\bx_i^T\bb_j\!\!>\!\!\bx_i^T\bb_k) 
\end{equation*}
where $\bb_j$ and $\bb_k$, respectively, denote the $j$-th and $k$-th columns of $\bB$. $B_{rs}$ appears in the sum when either $j=s$ or $k=s$. Simple calculations show that:
\begin{align*}
S_n(\bB) = \frac{1}{n\binom{q}{2}} \sum_{i=1}^n \sum_{j=1}^q \Big( &\bid(y_{ij}\!\!>\!\!y_{is})\bid(\bx_i^T\bb_j\!\!>\!\!\bx_i^T\bb_s) \\ &+\bid(y_{ij}\!\!<\!\!y_{is})\bid(\bx_i^T\bb_j\!\!<\!\!\bx_i^T\bb_s) \Big) + c,
\end{align*}
where $c$ denotes the sum of all terms that do not depend on $B_{rs}$. It is now easy to see that maximizing $S_n(\bB)$ with respect to $B_{rs}$ is an instance of~(\ref{eq:opt_simp}); $c$ disappears in the optimization, and depending on whether $y_{ij}>y_{is}$ or $y_{ij}<y_{is}$, only one of the  terms in each set of parentheses remains. In maximizing $S_n(\bB)$ with respect to $B_{rs}$, we have $T=nq$; if $y_{ij}>y_{is}$, $u_t=\bx_i^T(\bb_j-\bb_s)+x_{ir}B_{rs}$, and $v_t = -x_{rs}$; and if $y_{ij}<y_{is}$, $u_t=\bx_i^T(\bb_s-\bb_i)-x_{ir}B_{rs}$, and $v_t = x_{rs}$. Computing each $u_t$ and $v_t$ requires $O(p)$ operations, and thus computing all $u_t$ and $v_t$ for $t=1,\ldots,nq$ requires $O(nqp)$ operations. In total, maximizing $S_n(\bB)$ with respect to one of the elements of $\bB$ requires $O(nq(p+\log(nq))$ operations; $O(nqp)$ to compute the coefficients needed for formulating~(\ref{eq:opt_simp}) and $O(nq\log(nq))$ to solve it. 

In each full round of the alternating maximization procedure, we go over all $pq$ elements of $\bB$. So, a full round requires $O(npq^2(p+\log(nq))$ operations. After each round we adjust $\bB$ (by subtracting its last column from all its columns and then normalizing it) so that $\bB\in\mathcal{B}$. As argued before, this adjustment does not change the objective function. We continue this alternating maximization, until a point when a full round over all elements does not increase $S_n(\bB)$---an indication of reaching a fixed point. Since $S_n(\bB)$ is the sum of $O(nq^2)$ step functions,  it takes $O(nq^2)$ different values. Therefore, our algorithm stops at most in $O(nq^2)$ steps (since $S_n$ is increased in each round). Putting all these together, we conclude that our whole algorithm requires $O(n^2pq^4(p+\log(nq))$ operations, which is polynomial in all three parameters. The pseudo-code of the algorithm is presented in Algorithm 2. 

Our algorithm is greedy and it may converge to a local maximum. We can alleviate this problem to some degree by starting the algorithm from different random initial points and choosing the best result. In the next section, we show that our proposed alternating maximization scheme is successful in providing very good estimates of the true coefficient matrix. 
\begin{algorithm}[!h]
\caption*{ \bf Algorithm 2 --- for solving~(\ref{eq:opt}) }
\begin{algorithmic}
\State {\bf Inputs}: $\bx_i, \by_i, i=1,\ldots,n$ ;  {\bf Output}: $\widehat{\bB}_n$
\State \texttt{Random Initialization:} $\bB \gets$ randn$(p,q)$
\While{$S_n(\bB)$ changes}
\For{$r=1,\ldots,p; s=1,\ldots,q$}
\State \texttt{Maximize $S_n(\bB)$ w.r.t. $B_{rs}$ with}
\State \texttt{other elements fixed using Alg.~1}
\EndFor
\State $\bb \gets \bB(:,q)$ \Comment{$\bb$ is the last column of $\bB$}
\State $\bB \gets (\bB-\bb\mathds{1}_{1\times q})/\|\bB-\bb\mathds{1}_{1\times q}\|_F$
\EndWhile

\State $\widehat{\bB}_n \gets \bB$
\end{algorithmic}
\end{algorithm}

\subsection{Imposing Sparsity}\label{sec:L0}
So far, we formulated the problem of finding $\bB$ with no extra structural constraint. In many settings (e.g., high-dimensional, high noise, limited data), solving these problems with no extra constraints results in poor predictive performance, often due to overfitting. Here, we explain how we can impose sparsity on the coefficient matrix in our alternating maximization scheme. The $\ell_0$ norm is the precise metric to measure sparsity, but is often relaxed to the $\ell_1$ norm  to make optimization problems convex. With our proposed formulation, it is possible to use the $\ell_0$ norm directly. Let us revisit the optimization problem in~(\ref{eq:opt_simp}), but this time with  an additional $\ell_0$ penalty on $x$. We have
\begin{equation}\label{eq:opt_spa}
\widehat{x} =  \argmax_{x \in \mathbb{R}}\qquad  \underbrace{ \sum_{t=1}^T \bid(u_t + v_tx>0) - \lambda \|x\|_0}_{f(x)},
\end{equation}
where $\lambda$ is the regularization parameter and $\|x\|_0 = 0$ if $x=0$ and $\|x\|_0 = 1$ if $x\neq0$. Solving this problem is very similar to solving the unconstrained problem in~(\ref{eq:opt_simp}). We only need to solve the problem without the constraint (say we get the solution $\widehat{x}$) and then compare the value of $f(\widehat{x})$ with $f(0)$. If $f(\widehat{x})>f(0)$, then we choose $\widehat{x}$ as the solution of the constrained problem and if $f(\widehat{x})\leq f(0)$, we choose $0$ as the solution. Thus, the solution of the constrained problem can be achieved with $O(1)$ extra computations compared to the unconstrained case. Therefore, we can impose element-wise sparsity on $\bB$ (i.e., having the penalty term $\lambda \sum_{i,j}\|B_{ij}\|_0$) by solving~(\ref{eq:opt_spa}) instead of~(\ref{eq:opt_simp}) in each step of the alternating maximization. 
\section{Simulation Study}\label{sec:sim1}
\subsection{Consistency of Optimizer}
In Section~\ref{sec:theory}, we proved that the global maximizer of~$S_n(\bB)$ is a consistent estimator of the coefficient matrix. However, $S_n(\bB)$ is a highly non-smooth objective function, and it is difficult to find its global maximum in general. Our proposed greedy algorithm in Section~\ref{sec:opt} is only guaranteed to reach a local maximum. In this section, by extensive simulations, we show that our alternating maximization is able to provide estimates very close to the true matrix for large values of $n$ which is in conformity with consistency. 

We use the model in~(\ref{eq:matmodel}) with various utility functions and noise distributions to generate the synthetic data. For fixed $p$ and $q$ (for the simulations in this section, we set $p=q=5$), we change $n$ from $2^3$ to $2^{17}$ in powers of two, and test the consistency of our alternating maximization scheme, i.e., how close the estimated coefficient matrix, $\widetilde{\bB}_n$, gets to the true one in the model, $\bB^*$.  To measure the similarity of $\widetilde{\bB}_n$ and $\bB^*$, we use the following two measures:
\begin{align*}
M_1(\widetilde{\bB}_n, \bB^*) & = \|\widetilde{\bB}_n - \bB^*\|_F^2 = \sum_{i,j} (\widetilde{B}_{ij} - B^*_{ij})^2, \\
M_2(\widetilde{\bB}_n, \bB^*) & = \frac{\sum_{i=1}^p\sum_{j=1}^q (\widetilde{B}_{ij} - \widetilde{b})(B^*_{ij} - b^*)}{\sqrt{\sum_{i,j} (\widetilde{B}_{ij} - \widetilde{b})^2}\sqrt{\sum_{i,j} (B^*_{ij} - b^*)^2}},
\end{align*}
where $\widetilde{b} = \sum_{i,j}\widetilde{B}_{ij}/pq \text{ and } b^* = \sum_{i,j} B^*_{ij}/pq$. Thus, $M_1$ measures how close the elements of the estimated coefficient matrix are to the true elements, and $M_2$ measures how correlated the elements of the two matrices are. 

We generate the data as follows. First, we generate an $n\times p$ predictor matrix, $\bX$, with  rows independently drawn from $\mathcal{N}({\bf 0}, \bsig_X)$, where the $(i,j)$-th element of $\bsig_X$ is defined as $\sigma^X_{i,j} = 0.7^{|j-i|}$. This is a common model for predictors in the  literature~\citep{Yua07, pen10,Roth10}. The true coefficient matrix, $\bB^*$, has $75\%$ non-zero elements which are drawn independently from a normal distribution $\mathcal{N}(0,1)$. To satisfy condition {\bf (C4)}, we then subtract its last column from all columns and normalize it so that $\bB^*\in\mathcal{B}$.

We consider three types of noise:
\begin{itemize}
\item[{\bf (E1)}]The elements of $\bE$ are drawn independently from a Gaussian distribution $\mathcal{N}(0,1)$.
\item[{\bf (E2)}] The elements of $\bE$ are drawn independently from a $t$-student distribution with $\nu=1$.
\item[{\bf (E3)}] The elements of $\bE$ are drawn independently according to a Gaussian mixture model; each element is drawn from a $\mathcal{N}(0,.2)$ with probability $.8$ and from a $\mathcal{N}(1,.2)$ with probability $.2$. 
\end{itemize}
After sampling, to have a consistent signal-to-noise ratio in various scenarios, the elements of $\bE$ are scaled accordingly to satisfy $\|\bE\|_F/\|\bX\bB^*\|_F=0.2$; i.e., the norm of the noise matrix is $20\%$ of the norm of the signal matrix. {\bf E1} corresponds to a general setting where the noise is Gaussian; {\bf E2} simulates a heavy-tailed noise which is present in many practical settings; and {\bf E3} simulates a case where $20\%$ of the data points are corrupted with larger noise (i.e., simulating outliers). 

We also consider three different utility functions:
\begin{itemize}
\item[{\bf (U1)}] Identity: $U(x) = x$.
\item[{\bf (U2)}] Sigmoid: $U(x) = 1/(1+e^{-x/5})$.
\item[{\bf (U3)}] Piecewise constant: $U(x) = \left \lfloor{x}\right \rfloor $.
\end{itemize}
The identity utility function transforms the problem in~(\ref{eq:matmodel}) to a regular multivariate regression problem. The sigmoid utility function has the nice property of diminishing marginal returns and is of special interest in many applications including economics~\citep{Fri48,Tve92} and network utility maximization~\citep{Faz05}. Piecewise constant functions are useful in choice models~\citep[Chapter 17]{Gre11} and also can be used in modeling the settings where the responses are ordinal, e.g., surveys asking subjects' levels of preference (say $1$--$10$ corresponding to ``very poor'' to ``excellent'') for a set of responses. 

Finally, we set $\bY = U(\bX\bB^*+\bE)$, where $U$ is applied individually to the elements of its input. For any given $\bX$ and $\bY$, we run Algorithm~2 described in Section~\ref{sec:opt} with 10 different starting points and choose the result with the highest objective value. The starting points are chosen by drawing elements independently from a standard normal distribution. For each $n$, we generate ten sets of $\bX$ and $\bY$ as described above and report the median of each of the two performance measures over these ten sets (mean results are similar). The results are shown in Figure~1.

As $n$ becomes larger, the sum-of-squares error, $M_1$, goes to 0 and the correlation, $M_2$, goes to 1. Therefore, in these cases, our greedy algorithm provides estimates with increasing similarity to the true matrix which is in agreement with consistency. We are increasing $n$ to very large values to show this consistency; however, the algorithm provides very good solutions for all practical purposes at smaller values of $n$. For instance, for {\bf (U2)} and $n=128$, the averages of $M_1$ and $M_2$ medians for all three noise settings are respectively $0.052$ and $0.95$. Since $\|\bB^*\|_F^2=1$, this means that the sum of all squared errors is only $5.2\%$ of the sum of the squares of the true elements, and that the estimated and true coefficients are highly correlated.

\begin{figure*}[!b]
     \begin{center}
        \subfigure{%
            \includegraphics[scale=0.19]{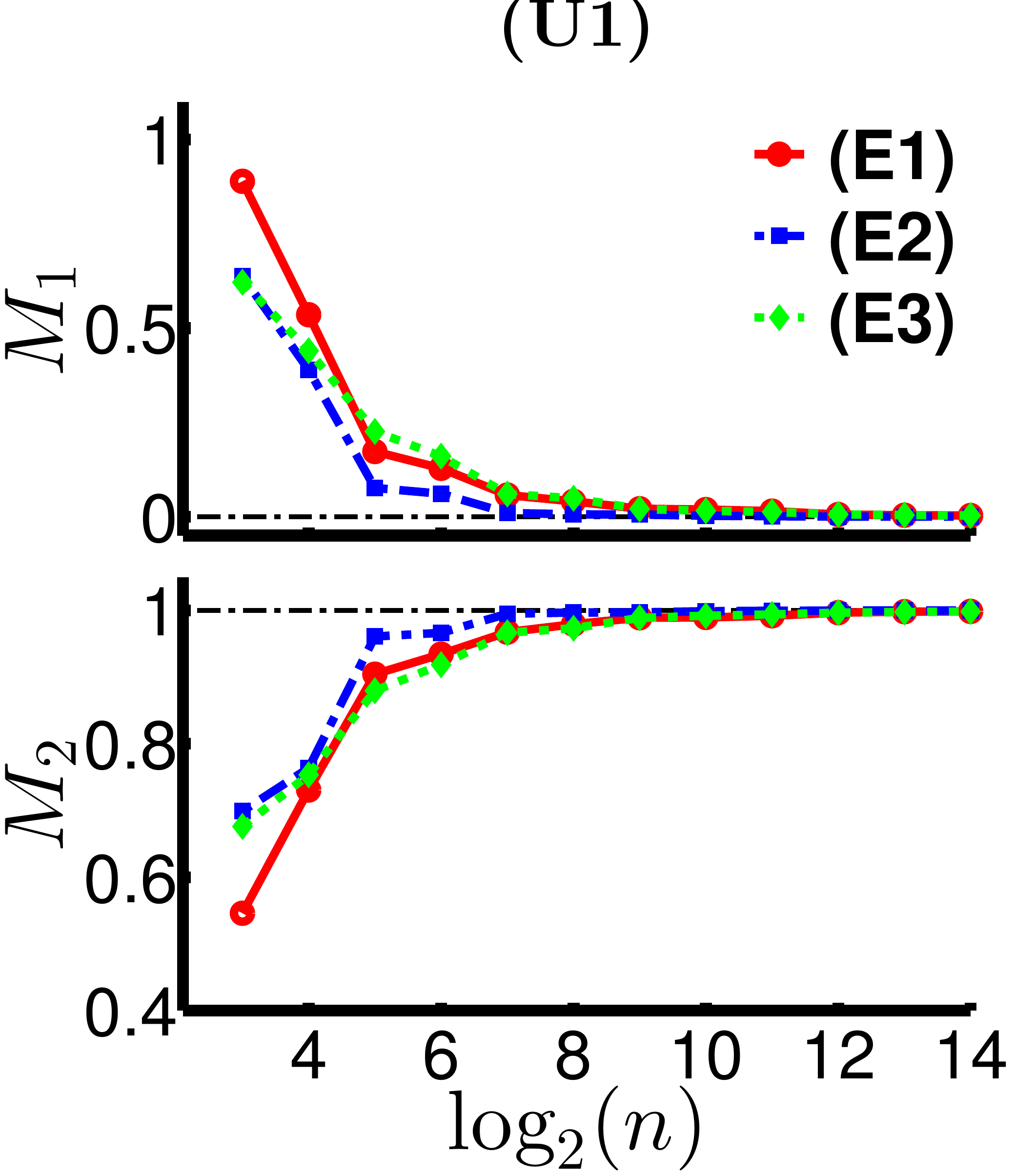}
        }%
        \subfigure{%
           \includegraphics[scale=0.19]{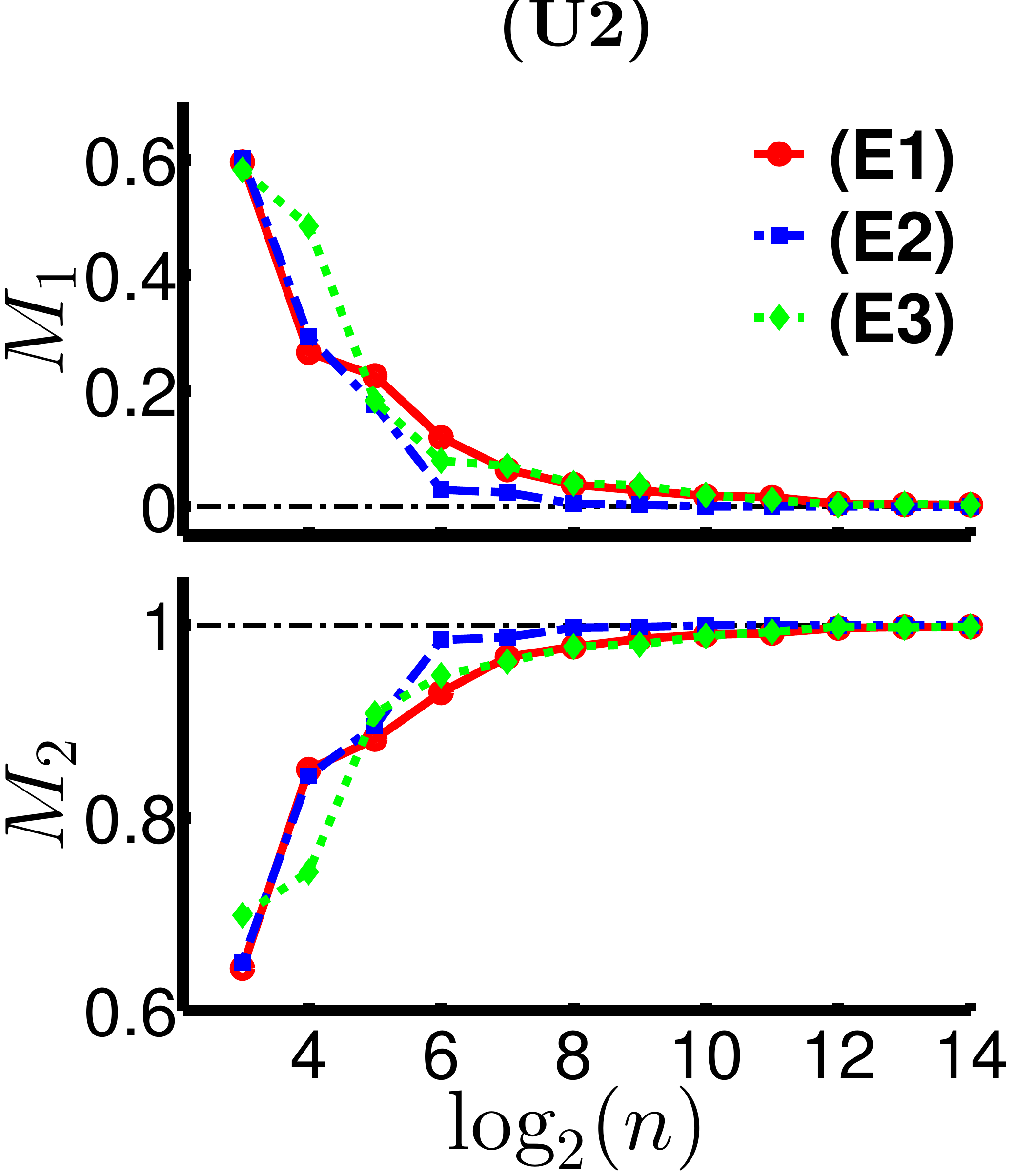}
        } 
\subfigure{%
           \includegraphics[scale=0.19]{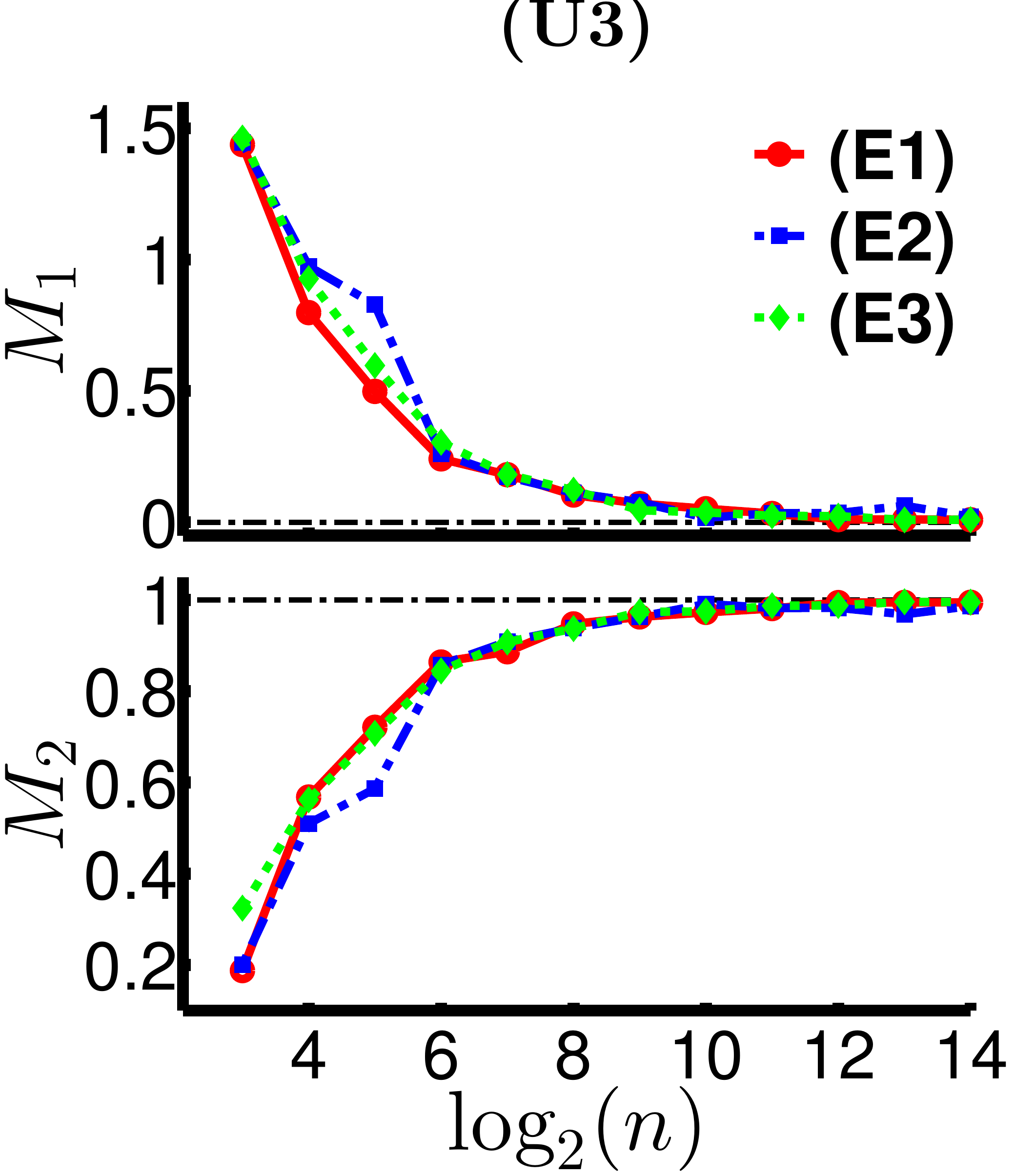}
        } 
    \end{center}
    \caption{Examining the consistency of the alternating maximization scheme for different noise distributions and utility functions. For both similarity metrics, medians over 10 runs are reported.}
   \label{fig:pm}
\end{figure*}

\subsection{Variable Selection}
The results in Figure~1 show that Algorithm~2 provides consistent estimates of $\bB^*$ in different settings; if an element of $\bB^*$ is zero, then the corresponding estimated element becomes very close to 0 for large values of $n$. However, in the absence of extra constraints, these estimates may not be exactly zero. Thus, to identify the relevant variables in a model, one has to use ad-hoc thresholding techniques, which is not desirable. The penalized formulation described in Section~\ref{sec:L0}, imposes sparsity on the elements of the estimated matrix, resulting in joint variable selection and estimation. In Table~1, we compare the variable selection and estimation of solutions of the penalized version of~(\ref{eq:opt}) with those of the unconstrained problem, for $n=128$ and the sigmoid utility function. Data is generated as before except $\bB^*$ has $50\%$ non-zero elements. For the penalized version, we set $\lambda=5$. We measure the variable selection performance by the following two metrics:
\vspace*{-0.02 in}
\begin{align*}
\textrm{Signed  Sensitivity}  & = \frac{\sum_{i,j}{\bf 1[} B^*_{i,j} \cdot \widehat{B}_{i,j} > 0{\bf ]}}{\sum_{i,j} {\bf 1}[B^*_{i,j}\neq 0]}, \\
\textrm{Specificity}  & = \frac{\sum_{i,j}{\bf 1[} B^*_{i,j} = 0{\bf ]}\cdot {\bf 1[} \widehat{B}_{i,j} = 0{\bf ]}}{\sum_{i,j} {\bf 1}[B^*_{i,j}= 0]}.
\end{align*}
\vspace*{-0.12 in}

We observe that although both versions of the algorithm produce solutions almost equally close to the true matrix (in terms of $M_1$), the penalized version provides  much better variable selection. In the original model, most of the elements are non-zero (some of them very close to zero), resulting in a poor specificity. The solutions of the penalized version have both high sensitivity and specificity.

\begin{table}[!h]
\centering
\begin{tabular}{|c|c||c|c|c|}
\hline 
 Problem & Metric & {\bf E1} & {\bf E2} & {\bf E3} \\ 
\hline 
\hline
\multirow{3}{*}{Original} & $M_1$ & 0.059  &  0.013 & 0.059  \\ 
\cline{2-5} 
 & Sensitivity & 1 & 1 & 1 \\ 
\cline{2-5}
 & Specificity & 0.36  &  0.37 &   0.34 \\ 
\hline \hline
\multirow{3}{*}{Penalized} & $M_1$ &  0.048  &  0.016   & 0.050 \\ 
\cline{2-5} 
 & Sensitivity & 0.89  &  0.92  &  0.87 \\ 
\cline{2-5}
 & Specificity &  0.71  &  0.87 &   0.78 \\ 
\hline 
\end{tabular} 
\caption{Variable selection comparison. For all metrics, medians over 20 runs are reported.}
\end{table}

\subsection{Comparison with Other Multivariate Regressions Algorithms}
Compared with traditional linear multivariate regression algorithms (such as LASSO), our method has two important differences. First, the objective function is defined as the Kendall's rank correlation between responses and their estimates rather than the sum-of-squares error (with some regularization). The second difference is in the assumed underlying model. We consider a model where responses are related to the predictors through an unknown, potentially non-linear function, $U$, whereas in the traditional techniques, the relationship is assumed to be linear (or in the generalized regression setting, a known function, e.g., $\log$, of inputs). 

When the underlying problem structure is similar to~(\ref{eq:matmodel}), and the objective is the rank correlation, it is expected that our algorithm performs better than traditional models, since it matches the structure and objective better. To verify this expected result, we perform an experiment on simulated data and compare our algorithm with some the state-of-the-art multivariate regression techniques. 

The setup is as follows. We assume a model as in~(\ref{eq:matmodel}) with 10 predictors and 10 responses. The elements of the predictor matrix are sampled from a uniform distribution between 0 and 100. Matrix $\bB$ is sparse with density $60\%$ and its elements are drawn from a standard uniform distribution. The utility function is a sigmoid: $U(x) = 1/(1+e^{-x})$. We consider three noise settings as in {\bf (E1)--(E3)}. The learning is done over 30 instances and the test is done on a separate set of 20 instances. For all algorithms, the regularization parameters are achieved via 5--fold cross-validation. The algorithms against which we compare our method are Least Squares (LS), a robust version of LS called Least Trimmed Squares (LTS)~\citep{vic00}, LASSO~\citep{Tib11}, Sparse Reduced Rank Regression (SRRR)~\citep{Chen12}, elastic net~\citep{Zou05enet}, and regressions with ridge regularization. 

We run each experiment 100 times and report the median and $95\%$ confidence intervals of the improvements in the test rank correlation in Table~\ref{tab:mvcomp}. In essence, we run a paired hypothesis test comparing our algorithm against each of the algorithms in Table~\ref{tab:mvcomp}, and report the median, $2.5$'th percentile, and $97.5$'th percentile of the test statistic, $c_1 - c_2$, where $c_1$ and $c_2$ are respectively the test rank correlations of our algorithm and the test rank correlation of the other algorithm. We observe that in all cases, our algorithm performs statistically significantly better than other algorithms.

\begin{table}[!h]
\hspace*{-0.1 in}
\centering
\begin{tabular}{|c|c|c|c|c|c|c|c|}
\hline 
 \multicolumn{2}{|c|}{Improvement over} & LS &  LTS & LASSO & SRRR & ElasticNet &  Ridge \\ 
\hline \hline
\multirow{ 2}{*}{E1} & median  & 0.11  & 0.18  & 0.09  & 0.57  & 0.52    & 0.50\\ 
\cline{2-8} 
  &  $95\%$ CI & [0.02, 0.30]  & [0.09, 0.51]  &  [0.01, 0.26] & [0.16, 1.2]  &  [0.27, 0.86] & [0.24, 0.93]   \\ 
\hline \hline
\multirow{ 2}{*}{E2}  &  median &  0.11 & 0.16  & 0.09 & 0.62 & 0.54  &  0.51    \\ 
\cline{2-8} 
  &  $95\%$ CI & [0.05, 0.41]  & [0.07, 0.37]  & [0.04, 0.34]  & [0.24, 1.2]  &  [0.27, 0.90] &  [0.24,0.96] \\ 
\hline \hline
 \multirow{ 2}{*}{E3} & median  & 0.11  &  0.17 & 0.08&    0.59 & 0.48 &   0.50   \\ 
\cline{2-8} 
  &  $95\%$ CI &  [0.01, 0.32] & [0.07, 0.44]  & [0.00, 0.29] & [0.21, 1.22]  & [0.23, 0.93]  &   [0.19, 0.88] \\ 
\hline 
\end{tabular} 
\caption{Improvements achieved over state-of-the-art multivariate regression algorithms by using our proposed rank-based semi-parametric multivariate regression. The objective is the rank correlation between estimates and true values of responses. We report the statistics for $c_1 - c_2$, where $c_1$ and $c_2$ are respectively the test rank correlations of our algorithm and the test rank correlation of the other algorithm. \label{tab:mvcomp}}
\end{table}

\section{Application to Real Data} \label{sec:real}
In this section, we study two problems with real data where the training set is in the form of an ordering and/or the objective is to order a set of items. 
\subsection{Sushi Dataset}
First, we consider the sushi preference dataset\footnote{http://www.kamishima.net/sushi/}. This dataset includes the preference ordering of 10 sushi types\footnote{Ebi, Anago, Maguro, Ika, Uni, Sake, Tamago, Toro, Tekka--maki, and Kappa--maki} by 5000 users and the demographic information about these users. For each user we keep the following features: gender (male/female), age group (15--19, 20--29, 30--39, 40--49, 50--59, 60$+$), region in which the user had lived for the longest period until 15 years old (11 in total), and the region in which the user currently lives (11 in total). We represent each feature with a binary indicator vector. For example, for the gender, we use $(0,1)$ for males and $(1,0)$ for females. We show different age groups by $(1,0,0,0,0)$, $(0,1,0,0,0,0)$, $\ldots$, $(0,0,0,0,1)$. Similarly, the regions are represented by binary vectors of size 11. Thus, in total, each user has a feature vector of size 30 representing his/her demographic information.

The goal of our prediction task is to estimate the ordering of the 10 sushi types for a new user only based on his/her demographic information. Note that neither collaborative filtering nor content--based filtering is applicable to this problem, since first, the new user may not have rated or ranked any sushi types beforehand and second, we do not include extra domain knowledge about each sushi type. On the other hand, our regression--based framework is suitable for this prediction task. From the 5000 users, we choose 2500 of them at random as the training set and keep the the rest as the test set. We take the average Kendall correlation between the rows of predicted and true orderings for the users in the test set as the performance metric. We repeat this random division of users into training and testing groups 100 times to achieve bootstrapped confidence intervals for the performance metric. 

In order to compare our algorithm to other regression--based algorithms, we need to transform the ordering into ratings. We use the technique described in~\citep{Kam10} and assign the ratings $1/11$, $2/11$, $3/11, \ldots, 10/11$ to the least preferred to most preferred items. Then, for the algorithms that performed well in the simulation study (see Table~\ref{tab:mvcomp}), we follow the same training and testing procedure as explained above. We also compare the results to a $K$ nearest neighbor technique (KNN) where the feature vector of a new user is compared with the available users to identify the $K$ most similar users (in terms of the Euclidean distance), and then its ratings are calculated by averaging the ratings of those $K$ neighbors. The parameters of all these models are found via 5--fold cross--validation. 
The results are shown in Table~\ref{tab:sushi}. As we observe, our algorithm outperforms other with high statistical significance. 
\begin{table}[!h]
\centering
\begin{tabular}{|c|c|c|c|c|c|c|}
\hline 
  & Order-based  &   LASSO &  LS& SRRR & KNN \\ 
\hline 
 Median & $0.34$  &     $0.31$ & $0.31$ & $0.18$ &  $0.31$\\ 
\hline 
 $95\%$ CI &  $[0.33, 0.35]$  &  $[0.30, 0.32]$ &  $[0.30, 0.32]$ & $[0.02, 0.27]$ & $[0.31, 0.32]$ \\ 
\hline 
\end{tabular} 
\caption{Comparison of median and $95\%$ confidence intervals of the performance metric (average Kendall correlations between rows of estimated and true responses).  \label{tab:sushi} }
\end{table}

\subsection{BIXI Dataset}
We use a dataset providing information about Montreal's bicycle sharing system called BIXI. The data contains the number of available bikes in each of the 400 installed stations for every minute. We use the data collected for the first three weeks of June 2012. From this dataset we first form the features as follows. We allocate two features to each station corresponding to the number of arrivals and departures of bikes to or from that station for every hour. We define two learning tasks: using the number of arrivals and departures in the last hour, estimate the ordering of the number of arrivals and the ordering of the number of departures in the next hour. The result of this estimation is useful in identifying the stations with the highest incoming or outgoing traffic; the BIXI management team can then provide these stations with more bikes or remove the extra bikes accordingly.  

Mathematically, we want to estimate $\bB$ such that the ordering of elements in the rows of $\bY$ and $\bX\bB$ are as similar as possible, where  $\bX_{t,j}$ and $\bX_{t,400+j}$ respectively show the number of arrivals and departures in hour $t$ at station $j$ and $\bY_{t,j}$ either show the number of arrivals or departures in hour $t+1$ at station $j$.

From the roughly 500 hourly data points during the first three weeks of June, we use the first 300 data points for training and use the other 200 points for testing. We compare our algorithm to the multivariate regression algorithms that performed well in Table~\ref{tab:mvcomp} (LS and LTS cannot be used since we have more features than instances and are replaced by ridge regression as the baseline) and a KNN algorithm.  All the parameters of the algorithms are found via 5--fold cross--validation. In Table~\ref{tab:bixi}, we compare the average Kendall correlations between corresponding rows of estimated and true response matrices. We observe that our algorithm provides a better estimate (by around $10\%$ to $20\%$) of the ordering of the stations in terms of the number of arrivals and departures in the next hour.

\begin{table}[!h]
\centering
\begin{tabular}{|c|c|c|c|c|c|}
\hline 
  & Order-based  &   LASSO & SRRR   &  Ridge & KNN \\ 
\hline 
 Arrivals & $0.39$  &     $0.33$ & $0.32$ &  $0.31$ & $0.36$\\ 
\hline 
 Departures &  $0.40$  &  $0.33$ & $0.33$ &  $0.29$ & $0.35$ \\ 
\hline 
\end{tabular} 
\caption{Comparison of Kendall correlations between rows of estimated and true responses.  \label{tab:bixi} }
\end{table}

\section{Conclusions}
\vspace*{-0.08 in}
In this paper, we considered a generalized regression problem where the responses are monotonic functions of a linear transformation of the predictors. We proposed a semi-parametric method based on rank correlation, which is invariant with respect to the functional form of the underlying monotonic function, to estimate the linear transformation. We showed that the solution to our formulated problem is a consistent estimator of the true matrix and identified the convergence rate. To find the solution, we need to maximize a highly non-smooth function. We proposed a greedy algorithm to solve that problem, and showed its success in estimating the true coefficient matrix through simulations over a variety of noise distributions and utility functions. Finally, we presented a penalized version of our problem which has the same computational complexity as the original problem, but results in  better variable selection and more interpretable models.

\appendix
\section{Proof of Lemma 3}

\begin{mydef}
Assume that $\bz$ is a random vector with a distribution over $\mathbb{R}^d$. For a given $\bth$, we say $\mathcal{F}=\{f(\bz,\bthe), \bthe\in\bth\}$  is a $P$--degenerate class of functions over $\mathbb{R}^d$ if $E_{\bz}\left[f(\bz,\bthe)\right]=0$ for all $\bthe\in\bth$.
\end{mydef}

\begin{mydef}[Definition (2.7) from \citep{Pakes}]
A class of functions, $\mathcal{F}$, is called Euclidean for envelope $F$ if there exist constants $A$ and $V$ such that we have: if $0<\epsilon\leq 1$ and if $\mu$ is a measure such that $\int Fd\mu<\infty$, then there are functions $f_1,f_2,\ldots,f_k$ in $\mathcal{F}$ such that (i) $k\leq A\epsilon^{-V}$ and (ii) $\mathcal{F}$ is covered by the union of closed balls with radius $\epsilon \int Fd\mu$ and centers $f_1,\ldots,f_k$. In other words, for each $f\in\mathcal{F}$, there is an $f_i$ with $\int|f-f_i|d\mu\leq\epsilon\int Fd\mu$.  $A$ and $V$ must not depend on $\mu$.
\end{mydef}

\begin{mylemma}[Example (2.11) from \citep{Pakes}]\label{lem:pw} The class of piece-wise constant functions that are bounded by a fixed function $F$ is Euclidean with envelope $F$. 
\end{mylemma}

\begin{mylemma}[Corollaries 17 and 21 from \citep{Nolan}] \label{lem:corol}
~\\ (i) If $\mathcal{F}$ is Euclidean for envelope $F$ and $\mathcal{G}$ is Euclidean for envelope $G$, then $\mathcal{F}+\mathcal{G}$ is Euclidean with envelope $F+G$. \\ (ii) If $\mathcal{F}$ is a uniformly bounded class of functions, then for each finite measure $\nu$, the class $\nu\mathcal{F}$ is Euclidean. 
\end{mylemma}

\begin{mylemma}[Corollary 8 from \citep{Sherman}]\label{lem:rootnb}
Let $\bz$ be a $d$-dimensional random vector, $\mathcal{F}=\{f(\bz,\bthe), \bthe\in\bth\}$ be a class of $P$--degenerate functions over $\mathbb{R}^d$, and $\bthe_0$ be a point in $\bth$ for which $f(\bz,\bthe_0)=0$ for all $\bz$. If (i) $\mathcal{F}$ is Euclidean for an envelope $F$ satisfying $E[F^2]<\infty$ and (ii) $E_{\bz}[f(\bz,\bthe)]\to 0$ as $\bthe\to\bthe_0$, then uniformly over $o_p(1)$ neighborhoods of $\bthe_0$, for i.i.d. samples $\bz_i$,
\begin{equation}
\frac{1}{n}\sum_{i=1}^nf(\bz_i,\bthe) = o_p(1/\sqrt{n}).
\end{equation}
\end{mylemma}

\setcounter{mylemma}{2}
\begin{mylemma}\label{lem:euc}
For $\bthe$ in an $o_p(1)$ neighborhood of $\bthe_0$, and $S(\bthe) \triangleq E_{\bz}\left[h(\bz,\bthe)\right]$, we have:
\begin{equation}
S_n(\bthe) = S(\bthe) + S_n(\bthe_0) - S(\bthe_0) + o_p(1/\sqrt{n}).
\end{equation} 
\end{mylemma}
\begin{proof}
For this proof, we use results from~\citep{Nolan, Pakes, Sherman}. Define:
\begin{equation}
f(\bz,\bthe) \triangleq h(\bz,\bthe) - h(\bz,\bthe_0) -S(\bthe) + S(\bthe_0)
\end{equation}
Since $E_{\bz}\left[f(\bz,\bthe)\right]=S(\bthe) - S(\bthe_0) - S(\bthe) + S(\bthe_0)=0$, $\mathcal{F}=\{f(\bz,\bthe), \bthe\in\bth\}$ is $P$--degenerate. All the functions in $\mathcal{H}=\{h(\bz,\bthe) - h(\bz,\bthe_0), \bthe\in\bth\}$ are piecewise constant and thus $\mathcal{H}$ is Euclidean (Lemma~\ref{lem:pw}). 
Also, $h(\bz,\bthe)$ is uniformly bounded since $|h(\bz,\bthe)|<q^2$. 
Thus, $\{-S(\bthe) + S(\bthe_0), \bthe\in\bth\}$ is Euclidean, because it is the expected value of a Euclidean class of uniformly bounded  functions (Lemma~\ref{lem:corol}). 
Since the sum of two Euclidean classes is also Euclidean (Lemma~\ref{lem:corol}), we conclude that $\mathcal{F}$ is  Euclidean. An envelope for $\mathcal{F}$ is $4q^2$, since each of its four comprising summand functions is less than $q^2$.

We have shown that all the conditions of Lemma~\ref{lem:rootnb} hold: $\mathcal{F}$ is a Euclidean class of $P$-degenerate functions with a constant envelope; $f(\bz,\bthe_0)=0$; and $E_{\bz}\left[f(\bz,\bthe)\right]=0$. Application of the Lemma gives:
\begin{equation}
\frac{1}{n}\sum_{i=1}^nf(\bz_i,\bthe) = o_p(1/\sqrt{n}),
\end{equation}
for $\bthe$ in an $o_p(1)$ neighborhood of $\bthe_0$. Replacing $f$ with its summands gives the  result.
\end{proof}

\section*{Acknowledgement}
This work is supported by the Natural Sciences and Engineering Research Council of Canada (NSERC).
\bibliographystyle{IEEEtran}
\bibliography{corrBIB,SMFR,asympBib}
\end{document}